\documentclass{amsart}

\usepackage{xcolor,enumerate}       
 \definecolor{darkgreen}{rgb}{0,0.5,0}
\usepackage{blindtext}
\usepackage{mathtools}
\usepackage{tikz}
\usetikzlibrary{arrows.meta,calc,positioning,decorations.pathmorphing}

\usepackage[T1]{fontenc}    
\usepackage{hyperref}       
\usepackage{url}            
\usepackage{amsmath,amssymb}
\usepackage{amsmath,mathrsfs,bm}
\usepackage{parskip}
\usepackage{graphicx}
\usepackage{paralist}
\usepackage{hyperref}
\usepackage{bigints}

\newcommand{\oprocendsymbol}{\hbox{$\bullet$}}
\newcommand{\oprocend}{\relax\ifmmode\else\unskip\hfill\fi\oprocendsymbol}

\newtheorem{sublemma}{Sublemma}

\DeclareMathOperator{\rank}{rank}
\DeclareMathOperator{\spanop}{span}
\DeclareMathOperator{\Zeros}{Zeros}
\DeclareMathOperator{\Betti}{Betti}

\newcommand{\map}[3]{#1:#2 \rightarrow #3}

\newcommand{\real}{{\mathbb{R}}}

\newcommand{\locus}[2]{\mathcal{Z}_{#1,#2}}
\newcommand{\Aquad}[1]{\mathcal A_{\mathrm{quad},#1}}

\newtheorem{theorem}{Theorem}[section]

\newtheorem{lemma}[theorem]{Lemma}

\newtheorem{proposition}[theorem]{Proposition}

\newtheorem{definition}[theorem]{Definition}

\numberwithin{equation}{section}

\begin{document}

\begin{abstract}
We show that neural networks with activations satisfying a Riccati-type ordinary differential equation condition, an assumption arising in recent universal approximation results in the uniform topology, produce Pfaffian outputs on analytic domains with format controlled only by the architecture. Consequently, superlevel sets, as well as Lie bracket rank drop loci for neural network parameterized vector fields, admit architecture-only bounds on topological complexity, in particular on total Betti numbers, uniformly over all weights.
\end{abstract}

\title[]{On the Topology of Neural Network Superlevel Sets}

\thanks{The work of the author is supported by the Natural Sciences and Engineering Research Council of Canada.}

\author[Bahman Gharesifard ]{Bahman Gharesifard }
\address{Department of Mathematics and Statistics\\
Queen's University\\
Kingston, ON, Canada}
\email{bahman.gharesifard@queensu.ca}
\urladdr{https://gharesifard.github.io}

\maketitle

\section{Introduction}
\label{sec:intro}

Neural networks are routinely deployed in settings where the salient geometric object is not the raw scalar score
$F:\mathbb R^d\to\mathbb R$ itself, but rather a \emph{threshold region} (or \emph{superlevel set})
\[
\mathcal S_{\ge \tau}(F)\;:=\;\{x\in V:\ F(x)\ge \tau\},
\]
for a fixed threshold $\tau\in\mathbb R$ and a domain $V\subset\mathbb R^d$. Such sets encode global geometric information about the input--output map and can, a priori, exhibit rich topology, such as many connected components and higher-dimensional holes, even when $F$ is smooth.

A substantial literature studies neural-network capacity through geometric and combinatorial surrogates for decision complexity.
One classical line quantifies expressiveness by counting linear
regions and its dependence on depth and width \cite{montufar2014number,raghu2017expressive,naitzat2020topology},
while complementary works emphasize depth-driven oscillation and depth separation phenomena \cite{telgarsky2016benefits}.
A different viewpoint, advocated notably by~\cite{bianchini2014complexity,guss2018characterizing,ergen2024topological}, is to measure classifier complexity by topological invariants of decision regions and to compare shallow and deep
architectures through upper and lower bounds. Our work is closely aligned with this topological perspective, which 
questions whether $\mathcal S_{\ge \tau}(F)$ can become arbitrarily complicated. 

In contrast to the available results in the literature, however, we provide a structural viewpoint, stemming from universal approximation theory, that explains why these topology bounds can be made uniform over all weights once the architecture is fixed. More specifically, we show that for a broad class of smooth activations characterized by a \emph{Riccati-type differential constraint}, the network outputs on any fixed analytic domain lie in a tame Pfaffian class with \emph{format} controlled solely by architectural parameters (depth and widths) and the Riccati index. In line with~\cite{bianchini2014complexity}, the results yield explicit architecture-only bounds on the total Betti number, defined later in this document, of $\mathcal S_{\ge \tau}(F)$ on a fixed analytic domain.

The Riccati hypothesis plays a central role.  It requires that for some $r\in\mathbb N_0$ the $r$-th derivative of the
activation satisfies a quadratic differential equation of Riccati type on its domain of analyticity.  This condition is not ad
hoc: in~\cite{PT-BG:21,PT:BG:22-tac,PT-BG-24-c}, this quadratic differential-equation property on the activation function, or a finite derivative of it, is established as a sufficient structural assumption for universal approximation in the uniform topology for deep residual/flow models. It also turns out that a wide range of commonly used smooth activations satisfy this property.

A second contribution is that the same structural mechanism extends beyond scalar decision regions to geometric loci arising when neural networks parameterize \emph{vector fields}. This situation appears, for instance, when a learning or control update law produces a family of dynamics by varying the network weights, so that each parameter choice yields a different collection of vector fields on a fixed state space. Although we develop the relevant constructions in detail later, we briefly make the setting precise here to make a few points: let us fix $m$ real-analytic vector fields $X_1,\dots,X_m$ on $V\subset\mathbb R^d$, whose coordinate components are realized by networks of a fixed architecture, and let $\Delta^k(z)$ denote the span of all iterated Lie
brackets of $\{X_1,\dots,X_m\}$ of length $ k$. We consider the rank-drop strata
\[
\locus{k}{\rho}=\{z\in V:\ \dim\Delta^k(z)\le \rho\}.
\]
Such sets encode where bracket-generated directions fail to reach a prescribed dimension threshold and are standard in geometric
control~\cite{VJ:97} and sub-Riemannian geometry~\cite{AA-DB-UB:19}. In particular, they play a central role in capturing the set of output reachable by the neural network dynamics. 
We prove that, under the same Riccati activation hypothesis, the topology of these strata
also admits explicit weight-independent bounds depending only on $(d,m,k,\rho)$ and the network architecture.  To the best of the author's knowledge, weight-uniform bounds on Betti numbers for Lie-bracket rank-drop loci arising from neural network parameterized vector fields do not appear in the existing work.

At a technical level, the upshot of the proof is showing that on any analytic domain the network outputs (and, in the control
setting, the vector-field components and the Lie bracket expressions built from them) belong to the class of Pfaffian functions with
a \emph{format} depending only on the architecture and the Riccati index. Classical Pfaffian complexity theorems, similar to the ones used in~\cite{bianchini2014complexity}, then yield
the desired uniform bounds on zeros and Betti numbers~\cite{khovanskiui1991fewnomials,zell1999betti,gabrielov2004betti}.

\section{Neural Network Superlevel Sets: A Topological Problem}\label{sec:NN}
Consider a standard feedforward  neural network with input dimension
$n_0:=d\in\mathbb N$, where $\mathbb N$ is the set of positive integers, depth $L\in\mathbb N$, and layer widths
$n_1,\dots,n_L\in\mathbb N$. For each $\ell\in\{1,\dots,L\}$, let
\[
W^{(\ell)}\in\mathbb R^{n_\ell\times n_{\ell-1}},\qquad
b^{(\ell)}\in\mathbb R^{n_\ell}
\]
denote the weight matrix and bias vector of layer $\ell$. Let $\sigma:\mathbb R\to\mathbb R$
be an activation function, and define its componentwise extension
$\Sigma_\ell:\mathbb R^{n_\ell}\to\mathbb R^{n_\ell}$ by
\[
\Sigma_\ell(z):=\big(\sigma(z_1),\dots,\sigma(z_{n_\ell})\big).
\]
Given an input $x\in\mathbb R^d$, we set $h^{(0)}(x):=x\in\mathbb R^{n_0}$ and define, recursively,
for each $\ell\in\{1,\dots,L\}$,
\begin{equation}\label{eq:nn-vector}
h^{(\ell)}(x)=\Sigma_\ell\!\big(W^{(\ell)}h^{(\ell-1)}(x)+b^{(\ell)}\big)\in\mathbb R^{n_\ell}.
\end{equation}
For simplicity, we consider a scalar output map $F:\mathbb R^d\to\mathbb R$ of the form
\begin{equation}\label{eq:nn-output}
F(x):=c_0+c^\top h^{(L)}(x),
\end{equation}
where $c_0\in\mathbb R$ and $c\in\mathbb R^{n_L}$. The structural parts of what follows extend
componentwise to vector-valued outputs, and the statements below apply to scalar coordinates
or scalar linear combinations thereof.

We now focus on the class of activation functions which we consider in this work. In~\cite{PT:BG:22-tac}, and in the context of universal approximation of deep neural networks, modeled by neural ODEs, the following class of activation functions were shown to be sufficient for universal approximation in uniform topology. We recall this class of activation functions next.
\begin{definition}
\label{def:aquad}
For $r\in\mathbb N_0$, we say that $\sigma:\mathbb R\to\mathbb R$ belongs to $\mathcal A_{\mathrm{quad},r}$ if it is
nondecreasing and there exist constants $a_0,a_1,a_2\in\mathbb R$ with $a_2\neq 0$ such that, defining the Riccati ordinary differential equation \emph{(Riccati ODE)}:
\begin{equation}\label{eq:riccati}
\zeta'(t)=a_0+a_1\zeta(t)+a_2\zeta(t)^2,
\end{equation}
where $t$ is in the open set where the solution is defined, we have
\[
\zeta(t)=\frac{d^{\,r}\sigma}{dt^{\,r}}(t).
\]
We refer to $r$ as the \emph{Riccati index} of $\sigma\in\mathcal A_{\mathrm{quad},r}$.
\end{definition}
In other words, either $ \sigma $ or its $ r$th derivative satisfies the Riccati ODE~\eqref{eq:riccati}. 
Interestingly, a wide class of activation functions used in the literature meets this condition, including the logistic, hyperbolic tangent, and softplus functions. Moreover, other commonly used functions such as ReLU and GeLU can be well approximated within this class.

Clearly, the definition above requires that $\sigma\in C^{r+1} $. Note that solutions of real-analytic ordinary differential equations (ODEs), i.e., an ODE with real-analytic righthand side, are real-analytic on any open interval on which they exist as classical solutions, and anti-derivatives of real-analytic functions are real-analytic. Thus, on any open interval $I$ on which
\eqref{eq:riccati} holds and $\zeta$ is defined, i.e., there is no blow-up, the functions
$\zeta,\sigma^{(j-1)},\dots,\sigma$ are real-analytic on $I$. Unless stated otherwise, all results to follow are stated \emph{locally} on any open domain where all relevant expressions are analytic.

We now define one of the main objectives of interest in this work, postponing the more complicated setting of parametrized vector fields to Subsection~\ref{subsec:control-geom} for ease of read. 
\begin{definition}
\label{def:decision}
For real-valued function $F$ on a domain $V\subseteq\mathbb R^d$, the superlevel set is defined as 
\[
\mathcal{S}_{\ge 0}(F):= \{x\in V:\ F(x)\ge 0\}.
\]
\end{definition}

We now provide an informal statement of one of the problems of interest in this paper, postponing the more involved discussion on parametrized vector fields to later: 

{\bf Problem Statement:} We are interested in understanding the topological structure of the superlevel set of the scalar output function~\eqref{eq:nn-output} of a neural network. For instance, for the binary classification problem, where the \emph{sign} of the real-valued score $F$ to produce labels in $\{+1,-1\}$, we ask whether the set $\mathcal S_{\ge 0}(F)$ can have arbitrarily complicated structure with many connected components. 

\begin{figure}[t]
\centering
\begin{tikzpicture}[>=Latex, line cap=round, line join=round]
\draw[thick] (-6,0) -- (6,0) node[right] {$\mathbb R$};

\draw[thick] (0,-0.14) -- (0,0.14) node[below=4pt] {$0$};
\draw[line width=5pt, gray!25] (0.1,0) -- (6,0);
\draw[thick, rounded corners=18pt] (-5,1.2) rectangle (5,5.2);
\node at (0,5.55) {$V\subset\mathbb R^d$};

\draw[->, very thick] (0,1.05) -- (0,0.25) node[midway, right] {$F$};

\path[fill=gray!18]
  (-4.3,4.4) .. controls (-3.7,5.0) and (-2.8,4.9) .. (-2.2,4.3)
  .. controls (-1.7,3.8) and (-2.0,3.1) .. (-2.8,3.1)
  .. controls (-3.8,3.1) and (-4.6,3.7) .. (-4.3,4.4) -- cycle;
\draw[thick, dashed]
  (-4.3,4.4) .. controls (-3.7,5.0) and (-2.8,4.9) .. (-2.2,4.3)
  .. controls (-1.7,3.8) and (-2.0,3.1) .. (-2.8,3.1)
  .. controls (-3.8,3.1) and (-4.6,3.7) .. (-4.3,4.4) -- cycle;

\path[fill=gray!18]
  (-0.3,4.4) .. controls (0.9,5.2) and (2.8,5.0) .. (3.2,4.0)
  .. controls (3.6,3.1) and (2.4,2.2) .. (1.0,2.3)
  .. controls (-0.7,2.4) and (-1.2,3.6) .. (-0.3,4.4) -- cycle;
\draw[thick, dashed]
  (-0.3,4.4) .. controls (0.9,5.2) and (2.8,5.0) .. (3.2,4.0)
  .. controls (3.6,3.1) and (2.4,2.2) .. (1.0,2.3)
  .. controls (-0.7,2.4) and (-1.2,3.6) .. (-0.3,4.4) -- cycle;

\path[fill=white]
  (1.05,3.65) .. controls (1.55,4.00) and (2.25,3.95) .. (2.50,3.45)
  .. controls (2.75,2.95) and (2.15,2.55) .. (1.45,2.70)
  .. controls (0.80,2.85) and (0.65,3.35) .. (1.05,3.65) -- cycle;
\draw[thick, dashed]
  (1.05,3.65) .. controls (1.55,4.00) and (2.25,3.95) .. (2.50,3.45)
  .. controls (2.75,2.95) and (2.15,2.55) .. (1.45,2.70)
  .. controls (0.80,2.85) and (0.65,3.35) .. (1.05,3.65) -- cycle;

\path[fill=gray!18]
  (3.6,2.5) .. controls (4.4,2.9) and (4.9,2.3) .. (4.5,1.7)
  .. controls (4.1,1.2) and (3.1,1.4) .. (3.1,2.1)
  .. controls (3.1,2.35) and (3.3,2.4) .. (3.6,2.5) -- cycle;
\draw[thick, dashed]
  (3.6,2.5) .. controls (4.4,2.9) and (4.9,2.3) .. (4.5,1.7)
  .. controls (4.1,1.2) and (3.1,1.4) .. (3.1,2.1)
  .. controls (3.1,2.35) and (3.3,2.4) .. (3.6,2.5) -- cycle;

\end{tikzpicture}
\caption{A scalar map $F:V\to\mathbb R$ and its decision region $D=\{F\ge 0\}=F^{-1}([0,\infty))$ are depicted schematically. The dashed curves
represent the zero level sets.}
\label{fig:decision_bundle}
\end{figure}
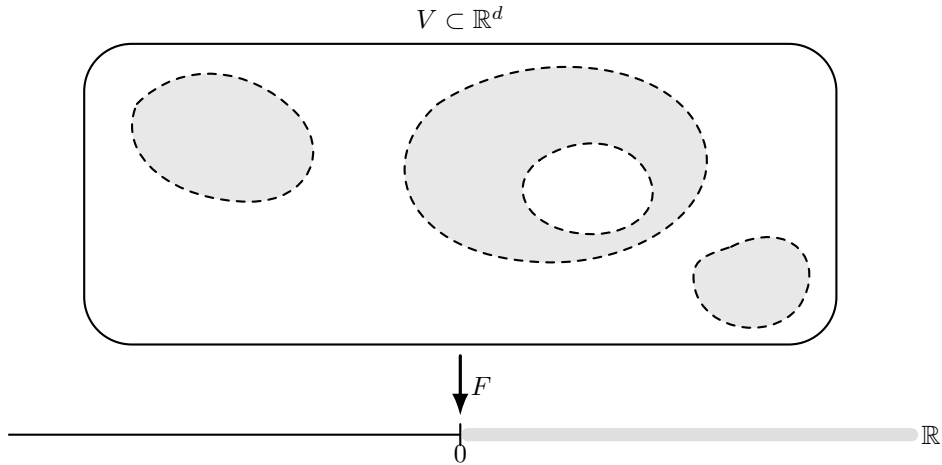

\section{Main Results}
We start by presenting a special instance when $ d=1 $, for reasons that the result is easier to explain and still sets the tone for the results to follow. Let 
$I\subset\real$ be an open interval. Given a real-analytic function $\map{F}{I}{\real} $ which is not identically zero, we denote by 
\[
\Zeros(F;I) := \big|\{x\in I:\ F(x)=0\}\big|,
\]
the number of distinct real zeros in $I$.

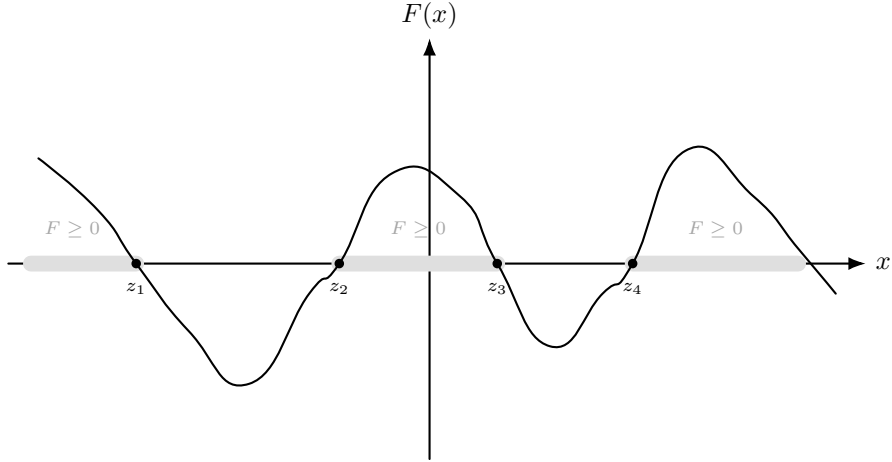
\begin{figure}[t]
\centering
\begin{tikzpicture}[>=Latex, line cap=round, line join=round]

\draw[->,thick] (-5.6,0) -- (5.8,0) node[right] {$x$};
\draw[->,thick] (0,-2.6) -- (0,3.0) node[above] {$F(x)$};

\draw[line width=6pt, gray!25, line cap=round] (-5.3,0) -- (-3.9,0);
\draw[line width=6pt, gray!25, line cap=round] (-1.2,0) -- (0.9,0);
\draw[line width=6pt, gray!25, line cap=round] (2.7,0) -- (4.9,0);

\node[gray!70] at (-4.75,0.45) {\scriptsize $F\ge 0$};
\node[gray!70] at (-0.15,0.45) {\scriptsize $F\ge 0$};
\node[gray!70] at (3.8,0.45) {\scriptsize $F\ge 0$};

\draw[thick]
  plot[smooth, tension=0.9] coordinates
  {(-5.2,1.4) (-4.4,0.7) (-3.9,0.0) (-3.2,-0.9)
   (-2.4,-1.6) (-1.6,-0.4) (-1.2,0.0) (-0.5,1.2)
   (0.4,0.9) (0.9,0.0) (1.6,-1.1) (2.3,-0.4)
   (2.7,0.0) (3.4,1.5) (4.3,0.9) (4.9,0.2) (5.4,-0.4)};

\foreach \x/\lab in {-3.9/$z_1$, -1.2/$z_2$, 0.9/$z_3$, 2.7/$z_4$} {
  \fill (\x,0) circle (1.8pt);
  \node[below=4pt] at (\x,0) {\scriptsize \lab};
}

\end{tikzpicture}
\caption{Assuming analyticity, each change of sign can only occur at a zero of $F$, so the number of interval components is controlled by $\Zeros(F;I)$.}
\label{fig:1d_zeros_intervals}
\end{figure}

\begin{proposition}
    \label{prop:deep_zeros_main}
    Suppose that $ d=1 $ and that $F:I\to\mathbb R$ be the output of a depth-$L$ neural network as in~\eqref{eq:nn-vector} with widths $n_1,\dots,n_L$ and $\sigma\in\Aquad{r}$, on an open interval $I$ where $F$ is analytic. Let
    \begin{equation}\label{eq:R}
    R := (r+2)\sum_{\ell=1}^{L} n_\ell.    
    \end{equation}    
    Then either $F\equiv 0$ on $I$, or 
\[
\Zeros(F;I)\ \le\ 2^{\frac{R(R+1)}{2}}C_I(1+L)^{R+1},
\]
    where $C_I>0$ is a constant depending on the interval $I$, but independent of the network weights and biases. In particular, the superlevel set $\mathcal{S}_{\ge 0}(F)$ is a union of at most
    $
 2^{\frac{R(R+1)}{2}}C_I(1+L)^{R+1}
$ intervals, uniformly over all weight values.
\end{proposition}

Some remarks are in order. In one dimension, a natural notion of classifier \emph{expressivity} is related to the number of label flips along the input line, i.e., the number of sign changes of~$F$. Several works (e.g.~\cite{raghu2017expressive}) quantify expressivity in terms of how strongly a network warps a one-dimensional input path, for instance, via the arc length of the induced representation trajectory.  In contrast, the one-dimensional result above directly controls a \emph{topological} notion of complexity along the input line: the number of label flips of $x\mapsto \mathrm{sign}(F(x))$, equivalently the number of zeros of $F$, and hence the number of interval components of $\{F\ge 0\}$. 
This complementary topological viewpoint has already been advocated in the literature; see, e.g.,~\cite{guss2018characterizing,ergen2024topological}, and should be viewed as a \emph{worst-case} expressivity bound; it upper-bounds the maximal possible oscillation over all parameter choices for a given architecture.

The bound above is \emph{uniform in the weights}: for a fixed architecture and $\sigma\in\Aquad{r}$, no choice of parameters can produce arbitrarily many sign changes on $I$. This uniformity also explains why the bound is expected to be large: it is required to hold simultaneously for \emph{every} parameter choice, rather than describing the typical behavior of a trained network. As stated in the introduction, the result shows that the type of topological control observed in~\cite{bianchini2014complexity} extends to all activation functions that satisfy the Riccati-type ODE~\eqref{eq:riccati} for some Riccati index $r$. As also noted in~\cite{bianchini2014complexity}, increasing the depth $L$ can substantially increase worst-case expressivity, since it enlarges the base of the exponential term.

At the risk of stating the obvious, we emphasize that this notion of complexity is different from statistical capacity measures such as VC dimension: here we control a \emph{global} geometric feature of the decision region rather than shattering of finite samples. Finally, the hypothesis that $F$ is real-analytic on $I$ is ensured if, for every layer $\ell$ and neuron $k$, we require that the quantities involved take values in an open interval on which $\sigma$ is real-analytic, i.e., away from any singularities or blow-up points of the Riccati solution (if any); we implicitly restrict to such intervals throughout.

Of course, the number of connected components does not capture all topological features of a superlevel set. We therefore use the Betti numbers, see Definition~\ref{def:betti}, which quantify the topology of a set across dimensions, measuring for instance the number of components and the presence of higher-dimensional holes.

We are now in a position to state a generalization of Proposition~\ref{prop:deep_zeros_main}. 
\begin{theorem}
\label{thm:nn_betti_s1}
Suppose that $d\ge 1$ and that $F:V\to\mathbb R$ is the output of a depth-$L$ neural network
as in~\eqref{eq:nn-vector}-\eqref{eq:nn-output} with widths $n_1,\dots,n_L$ and $\sigma\in\Aquad{r}$.
Assume $V\subset\mathbb R^d$ is an open domain on which $F$ is real-analytic, and let $R$ be as in~\eqref{eq:R}. 
Then 
\[
\Betti\!\big(\mathcal S_{\ge 0}(F)\big)\ \le\ \mathfrak B_V(d,R,L),
\]
where 
\[
\mathfrak B_V(d,R,L):=2^{\frac{R(R-1)}{2}}\; C_V\;
\Big(d+\min\{d,R\}(1+2L)\Big)^{d+R}.
\]
Here $C_V>0$ is a constant depending on the open domain $V$, but independent of the network
weights and biases. In particular, the number of connected components of the decision region is at most 
$ \mathfrak B_V(d,R,L)$. 
\end{theorem}

\subsection{Control Geometry}\label{subsec:control-geom}
In many settings of interest, neural networks are used to parameterize
vector fields (or controlled dynamics) on a state space. Concretely, fixing an architecture and varying the weights yields
a \emph{family} of vector fields. For example, one may have $m$ vector fields $X_1,\dots,X_m$ on $V\subset\mathbb R^d$ whose
coordinate functions $X_{i,p}$ are realized by neural networks; this for instance is instrumental when one studies universal
approximation properties~\cite{PT-BG:21}. A basic geometric object attached to such a family is the span of the fields together with their iterated Lie brackets: at each point $z$, the bracket span $\Delta^k(z)$ captures the directions that can be generated by combining the vector fields up to bracket length $k$. The dimension of $\Delta^k(z)$ is a local measure of how many directions the system can produce near $z$, and the rank-drop locus $\locus{k}{\rho}$ records where this span fails to reach a prescribed dimension threshold.

{\bf Relation to the superlevel-set viewpoint.}
In the classification setting studied earlier, one quantifies worst-case classifier complexity via the topology of a threshold
region defined by a single scalar score function $F$. In the control setting considered here, the corresponding geometric
objects are the rank-drop loci
\begin{equation}\label{eq:locusrho}
\locus{k}{\rho}=\{z\in V:\ \dim\Delta^k(z)\le \rho\},
\end{equation}
where $\Delta^k(z)$ is the span of all Lie brackets of the vector fields $X_1,\dots,X_m$ of length at most $k$. In this sense,
the viewpoint remains the same, but the measured quantity changes: we measure the failure, at a prescribed bracket depth $k$, to
generate sufficiently many independent directions.

Let $V\subset\mathbb R^d$ be open and let $X_1,\dots,X_m$ be real-analytic vector fields on $V$,
\[
X_i(z)=\sum_{p=1}^{d} X_{i,p}(z)\,\partial_{z_p}.
\]
Fix $k\in\mathbb N$ and let $\mathcal B_k$ be the finite set of all iterated Lie brackets of length up to $k$
built from $\{X_1,\dots,X_m\}$. Define
\[
\Delta^k(z):=\spanop\{Y(z):\ Y\in\mathcal B_k\}\subset \mathbb R^d,
\]
and let $A_k(z)$ be the $d\times |\mathcal B_k|$ matrix whose columns are the coordinate vectors of $Y(z)$.
For $\rho\in\{0,1,\dots,d\}$, the rank-drop locus in~\eqref{eq:locusrho} can be written as
\[
\locus{k}{\rho}=\{z\in V:\ \rank A_k(z)\le \rho\}.
\]
Let $M_1,\dots,M_{s_{k,\rho}}$ denote the list of all $(\rho+1)\times(\rho+1)$ minors of $A_k$, where
\[
s_{k,\rho}:=\binom{d}{\rho+1}\binom{|\mathcal B_k|}{\rho+1}.
\]

\begin{theorem}
\label{thm:nn_rankdrop_main}
Assume that each coefficient function $X_{i,p}:V\to\mathbb R$ is realized on $V$ by a feedforward neural network of fixed
architecture $(d,L,n_1,\dots,n_L)$ with activation $\sigma\in\Aquad{r}$, and that all such coefficient
functions are real-analytic on $V$. Set
\[
s_{k,\rho}:=\binom{d}{\rho+1}\binom{|\mathcal B_k|}{\rho+1}.
\]
Then
\[
\Betti(\locus{k}{\rho})\ \le\ \mathfrak R_V\!\Big(d,m,k,\rho;\,L,n_1,\dots,n_L,r\Big),
\]
where
\[
\mathfrak R_V\!\Big(d,m,k,\rho;\,L,n_1,\dots,n_L,r\Big)
:=
2^{\frac{R_k(R_k-1)}{2}}\; (s_{k,\rho})^{d}\; C_V\;\Big(d\beta_k+\min\{d,R_k\}\alpha_k\Big)^{d+R_k},
\]
where $C_V>0$ is a constant depending on the fixed domain $V$,
independent of all network weights and biases, and $R_k,\alpha_k,\beta_k$ are constants uniformly valid for all minors $M_i$,
$i\in \{1,\ldots,s_{k,\rho}\}$.
\end{theorem}

Theorem~\ref{thm:nn_rankdrop_main} provides weight-independent control of the geometric complexity of the rank-drop strata
$\locus{k}{\rho}$. Note that since
$\mathcal B_k\subseteq\mathcal B_{k+1}$, one has $\Delta^k(z)\subseteq\Delta^{k+1}(z)$ for all $z\in V$, and therefore
\begin{equation}\label{eq:locus-nested}
\locus{k+1}{\rho}\subseteq \locus{k}{\rho},\qquad\text{for all }k,\rho,
\end{equation}
so increasing $k$ can only shrink the degeneracy set. The parameter $k$ should be viewed as a chosen bracket depth at which one
measures rank generation; however, different weight choices may reach full rank at different depths.

\section{Proof of main results}\label{sec:proof-main}

We prove the main results of this paper, after providing some mathematical preliminaries

\subsection{Mathematical Preliminaries}
\label{subsec:aux-pfaff}

The key idea behind our main proofs comes from the fact that under our assumptions on the class of activation functions, the neural network output functions behave in a tame manner, in that they fall into the class of  ``Pfaffian functions'', which we recall shortly, see~\cite{khovanskiui1991fewnomials,gabrielov2004betti,zell1999betti}. We then rely on some classical results on the topological properties of superlevel sets defined by this class of functions to conclude the results. We start with defining this class of functions. 

\begin{definition}
\label{def:pfaff}
Let $V\subset\mathbb R^d$ be an open subset. We say that $(f_1,\ldots,f_R)$, 
where each $ \map{f_i}{V}{\real} $ is real-analytic on $ V $, forms a \emph{Pfaffian chain of degree $\alpha$} if for every $ i \in \{1,\ldots, R\} $ and every coordinate index $p\in\{1,\dots,d\}$ there exists a polynomial
$P_{i,p}(x,y_1,\dots,y_i)$ of total degree at most $ \alpha$ such that for all $ x \in V $,
\[
\frac{\partial f_i}{\partial x_p}(x)=P_{i,p}\big(x,f_1(x),\dots,f_i(x)\big). 
\]
A real-analytic function $F:V\to\mathbb R$ is \emph{Pfaffian of format $(d,R,\alpha,\beta)$} if
\[
F(x)=Q\big(x,f_1(x),\dots,f_R(x)\big),
\]
for some Pfaffian chain $(f_1,\dots,f_R)$ of degree $\alpha$ and some polynomial $Q$ of degree at most~$\beta$.
\end{definition}

We provide an example for clarification. Let $V=\mathbb R$ and define $f_1(x)=e^x$. Then
\[
\frac{d f_1}{dx}(x)=e^x=f_1(x)
\]
and henceforth, $(f_1)$ is a Pfaffian chain of degree $\alpha=1$. Therefore, any polynomial
\[
F(x)=Q\big(x,f_1(x)\big)=Q(x,e^x)
\]
is Pfaffian on $\mathbb R$; for instance, $F(x)=x^2+e^x$ is Pfaffian of format $(1,1,1,2)$.

Euler's Gamma function is an example of a real-analytic function which is not Pfaffian on any open subinterval of the positive reals, see~\cite{freitag2021not}.

We are interested in subsets defined by a collection of Pfaffian functions, as we outline next. 
\begin{definition}
\label{def:semipf}
Let $V\subset\mathbb R^d$ be open. A set $S\subset V$ is called \emph{semi-Pfaffian on $V$} if there exist Pfaffian functions
$G_1,\dots,G_s:V\to\mathbb R$  such that $S$ can be obtained from sets of the form
\[
\{x\in V:\ G_i(x)=0\},\qquad \{x\in V:\ G_i(x)\ge 0\},\quad \mathrm{and} \quad \{x\in V:\ G_i(x)\le 0\},
\]
where $i\in \{1,\dots,s\}$, by taking finitely many unions and intersections.
\end{definition}
We point out here that the definition of semi-Pfaffian in the literature is more general than what is presented in here~\cite{zell1999betti, gabrielov2004betti}, which suffices for our purposes. One instance related to this work is that when $G:V\to\mathbb R$ is Pfaffian, where $V\subset\mathbb R^d$ is an open set,
the superlevel set $\{x\in V:\ G(x)\ge 0\}$ is semi-Pfaffian on $V$. 

Finally, we recall the definition of the Betti numbers of a topological space. We refer the reader to~\cite[Chapter~2]{AH:02} for a full account. 
\begin{definition}
\label{def:betti}
Let $S\subset\mathbb R^d$ be a topological space. For each $i\ge 0$, the $i$-th Betti number of $S$
with $\mathbb Z_2$ coefficients is
\[
b_i(S):=\dim_{\mathbb Z_2} H_i(S;\mathbb Z_2),
\]
where $ H_i(S;\mathbb Z_2) $ is the $i$th singular homology group. 
The \emph{total Betti number} is
\[
\Betti(S):=\sum_{i=0}^{d} b_i(S).
\]
\end{definition}

We recall the next important classical result from~\cite[Theorem~3.4]{gabrielov2004complexity}. 

\begin{theorem}\label{thm:pf-betti}
Let $V\subset\mathbb R^d$ be an open domain and let $S\subset V$ be a semi-Pfaffian set described using $s$ Pfaffian
functions on $V$ of format $(d,R,\alpha,\beta)$. Then
\[
\Betti(S)\ \le\ 2^{\frac{R(R-1)}{2}}\; s^{d}\; C_V\;
\Big(d\beta+\min\{d,R\}\alpha\Big)^{d+R},
\]
where $C_V>0$ is a constant depending only on the fixed domain $V$ and the Pfaffian chain,
and independent of the coefficients/parameters in the Pfaffian functions defining $S$.
\end{theorem}

\subsection{Main proofs}
We start with proving Proposition~\ref{prop:deep_zeros_main}. 
Given the preliminaries established in the last subsection, it is enough if we show that under our assumptions on the activation functions, the output function of a given neural network is Pfaffian. 

\begin{proposition}
\label{prop:deep-pfaff}
Let $\sigma\in\Aquad{r}$ with Riccati index $r$. Consider a depth-$L$
network \eqref{eq:nn-vector}--\eqref{eq:nn-output} with widths $n_1,\dots,n_L$ and scalar output $F$.
Let $V\subset\mathbb R^d$ be an open set on which all affine layer inputs
\[
s^{(\ell)}(x):=W^{(\ell)}h^{(\ell-1)}(x)+b^{(\ell)}\in\mathbb R^{n_\ell},\qquad \ell=1,\dots,L,
\]
take values in an open interval where $\sigma$ is real-analytic. Then $F$ is Pfaffian on $V$ with format
\begin{equation}\label{eq:format-proof}
(d,R,\alpha,\beta)=\Big(d,\ (r+2)\sum_{\ell=1}^{L}n_\ell,\ 1+2L,\ 1\Big).
\end{equation}
\end{proposition}

\begin{proof}
Fix $\ell\in\{1,\ldots,L\}$ and write the coordinates of $s^{(\ell)}(x)\in\mathbb R^{n_\ell}$ and
$h^{(\ell)}(x)\in\mathbb R^{n_\ell}$ as $s^{(\ell)}_k(x)$ and $h^{(\ell)}_k(x)$, respectively. For each neuron $(\ell,k)$ and
each integer $q\in\{0,1,\dots,r\}$, we define the auxiliary functions
\[
u^{(\ell)}_{k,q}(x)\;:=\;\frac{d^{\,q}\sigma}{dt^{\,q}}\Big|_{t=s^{(\ell)}_k(x)}\!,
\]
which for $q=0$ reads that
\[
u^{(\ell)}_{k,0}(x)=\sigma\big(s^{(\ell)}_k(x)\big)=h^{(\ell)}_k(x),
\]
where $r$ is the Riccati index.

To verify that $F$ is Pfaffian, we will explicitly construct a Pfaffian chain by listing, as scalar-valued functions on $V$, all affine layer inputs
$s^{(\ell)}_k(x)=(W^{(\ell)}h^{(\ell-1)}(x)+b^{(\ell)})_k$ together with the functions obtained by evaluating the derivatives of $\sigma$ at these scalars.
To start, for each layer $\ell\in\{1,\dots,L\}$ and each neuron index $k\in\{1,\dots,n_\ell\}$, we gather the associated functions into the block
\begin{equation}\label{eq:block-order-rewrite}
\mathcal F_{\ell,k}
:=
\bigl(
s^{(\ell)}_k,\ u^{(\ell)}_{k,r},\ u^{(\ell)}_{k,r-1},\dots,u^{(\ell)}_{k,0}
\bigr),
\end{equation}
which has length $r+2$. Note that each block $\mathcal F_{\ell,k}$ is simply a finite tuple of real-analytic functions associated with the neuron $(\ell,k)$.
We will order the scalar functions appearing in these blocks into a single list so that for each list element $f_i$ and each coordinate index $p$ the derivative $\partial_p f_i$
can be written as a polynomial in $(x,f_1,\dots,f_i)$, as required in the definition of a Pfaffian chain.

We now form the ordered list $(f_1,\dots,f_R)$ mentioned above by concatenating these blocks in the following lexicographic order:
we list all blocks in layer $1$ in increasing neuron index, then all blocks in layer $2$, and so on, up to layer $L$, i.e., 
\[
\mathcal F_{1,1},\dots,\mathcal F_{1,n_1},\ 
\mathcal F_{2,1},\dots,\mathcal F_{2,n_2},\ \dots,\ 
\mathcal F_{L,1},\dots,\mathcal F_{L,n_L}.
\]
Within each fixed block $\mathcal F_{\ell,k}$, we keep the internal order
\(
s^{(\ell)}_k \prec u^{(\ell)}_{k,r} \prec u^{(\ell)}_{k,r-1}\prec\cdots\prec u^{(\ell)}_{k,0}.
\)
With this convention, every function belonging to layer $\ell$ appears after all functions from layers $<\ell$, and within a
fixed $(\ell,k)$ the function $u^{(\ell)}_{k,q+1}$ appears earlier than $u^{(\ell)}_{k,q}$.

Since each block has length $r+2$, we have
$
R=(r+2)\sum_{\ell=1}^{L}n_\ell.
$
Let us now fix a coordinate index $p\in\{1,\dots,d\}$ and write $\partial_p:=\frac{\partial}{\partial x_p}$.
From the network recursion \eqref{eq:nn-vector}, we have that
\begin{equation}\label{eq:ds-proof-rewrite}
\partial_p s^{(\ell)}_k(x)
=
\sum_{m=1}^{n_{\ell-1}} W^{(\ell)}_{km}\,\partial_p h^{(\ell-1)}_m(x)
=
\sum_{m=1}^{n_{\ell-1}} W^{(\ell)}_{km}\,\partial_p u^{(\ell-1)}_{m,0}(x),
\end{equation}
and for $\ell=1$, we have $s^{(1)}_k(x)=\sum_{m=1}^{d}W^{(1)}_{km}x_m+b^{(1)}_k$, hence $\partial_p s^{(1)}_k(x)=W^{(1)}_{kp}$.

Moreover, for $q\in\{0,\dots,r-1\}$, 
\[
u^{(\ell)}_{k,q}(x)=\frac{d^{\,q}\sigma}{dt^{\,q}}\Big|_{t=s^{(\ell)}_k(x)}. 
\]
Using the chain rule now, we obtain 
\begin{equation}\label{eq:jet1-proof-rewrite}
\partial_p u^{(\ell)}_{k,q}(x)
=
\frac{d^{\,q+1}\sigma}{dt^{\,q+1}}\Big|_{t=s^{(\ell)}_k(x)}\;\partial_p s^{(\ell)}_k(x)
=
u^{(\ell)}_{k,q+1}(x)\,\partial_p s^{(\ell)}_k(x).
\end{equation}
For $q=r$, set $\zeta(t):=\frac{d^{\,r}\sigma}{dt^{\,r}}(t)$, so that $u^{(\ell)}_{k,r}(x)=\zeta(s^{(\ell)}_k(x))$.
By Definition~\ref{def:aquad}, $\zeta$ satisfies \eqref{eq:riccati}, i.e.
$\zeta'(t)=a_0+a_1\zeta(t)+a_2\zeta(t)^2$, hence again by the chain rule,
\begin{equation}\label{eq:jet2-proof-rewrite}
\partial_p u^{(\ell)}_{k,r}(x)
=
\zeta'\!\big(s^{(\ell)}_k(x)\big)\,\partial_p s^{(\ell)}_k(x)
=
\bigl(a_0+a_1u^{(\ell)}_{k,r}(x)+a_2(u^{(\ell)}_{k,r}(x))^2\bigr)\,\partial_p s^{(\ell)}_k(x).
\end{equation}
We now form the next claim, which will allow us to conclude the result. 
\begin{sublemma}\label{sublem:degree-triangular}
For each $\ell\in\{1,\dots,L\}$, let $D_\ell:=1+2\ell$. Then for every $p\in\{1,\dots,d\}$ and every list element $f_i$ whose
block $\mathcal F_{\ell',k'}$ in \eqref{eq:block-order-rewrite} belongs to one of the first $\ell$ layers, i.e., $\ell'\le \ell$,
the partial derivative $\partial_p f_i$ can be written as a polynomial in
$
(x,f_1,\dots,f_i)
$
of total degree at most $D_\ell$. In particular, with $\alpha:=D_L=1+2L$, the ordered list $(f_1,\dots,f_R)$ is a Pfaffian chain
on $V$ of degree $\alpha$.
\end{sublemma}
\noindent{\bf Proof of Sublemma.}
We prove the result by induction on $\ell$. We start with the base case. 

\noindent\emph{Case $\ell=1$.}
We have $\partial_p s^{(1)}_k(x)=W^{(1)}_{kp}$, a constant (degree $0$). By~\eqref{eq:jet1-proof-rewrite}, for $q<r$, we have that 
\[
\partial_p u^{(1)}_{k,q}(x)=u^{(1)}_{k,q+1}(x)\,W^{(1)}_{kp},
\]
which is polynomial of degree $1$ in $(x,f_1,\dots,f_i)$. Now using \eqref{eq:jet2-proof-rewrite} yields
\[
\partial_p u^{(1)}_{k,r}(x)
=
\bigl(a_0+a_1u^{(1)}_{k,r}(x)+a_2(u^{(1)}_{k,r}(x))^2\bigr)\,W^{(1)}_{kp},
\]
which is polynomial of degree $2$ in $(x,f_1,\dots,f_i)$. Hence all derivatives of list elements from the first layer are polynomials of degree at most $D_1=3$ in $(x,f_1,\dots,f_i)$. This proves the base case.

\noindent\emph{Induction step.} Assume the statement holds for $\ell-1\ge 1$, i.e., every $\partial_p f_j$ for list elements $f_j$ from the first $\ell-1$ layers is a polynomial in $(x,f_1,\dots,f_j)$ of degree at most $D_{\ell-1}$.

Let us fix a neuron $(\ell,k)$ in layer $\ell$. Since each $\partial_p u^{(\ell-1)}_{m,0}(x)$ is polynomial of degree at most $ D_{\ell-1}$ in earlier list elements by the induction hypothesis,~\eqref{eq:ds-proof-rewrite} expresses $\partial_p s^{(\ell)}_k(x)$ as a linear
combination of such polynomials. Therefore $\partial_p s^{(\ell)}_k(x)$ is itself a polynomial in earlier list elements of degree
at most $ D_{\ell-1}$.

Now, let us fix $q\in\{0,\dots,r-1\}$. By~\eqref{eq:jet1-proof-rewrite}, we have that 
\[
\partial_p u^{(\ell)}_{k,q}(x)=u^{(\ell)}_{k,q+1}(x)\,\partial_p s^{(\ell)}_k(x).
\]
By the ordering of the block $\mathcal F_{\ell,k}$ in \eqref{eq:block-order-rewrite}, the function $u^{(\ell)}_{k,q+1}$
appears earlier in the list than $u^{(\ell)}_{k,q}$, and we have already shown that $\partial_p s^{(\ell)}_k$ is a polynomial in
earlier list elements. Hence the right-hand side is a polynomial in $(x,f_1,\dots,f_i)$ whose degree is at most
\[
1+D_{\ell-1}=D_{\ell-1}+1\le D_\ell.
\]
Finally, using \eqref{eq:jet2-proof-rewrite}, we have that 
\[
\partial_p u^{(\ell)}_{k,r}(x)
=
\bigl(a_0+a_1u^{(\ell)}_{k,r}(x)+a_2(u^{(\ell)}_{k,r}(x))^2\bigr)\,\partial_p s^{(\ell)}_k(x),
\]
which is a polynomial factor of degree $2$ in the list element $u^{(\ell)}_{k,r}$ times a polynomial of degree at most $ D_{\ell-1}$
in earlier list elements, hence has total degree at most $D_{\ell-1}+2=D_\ell$ in $(x,f_1,\dots,f_i)$.
This finishes the inductive step of the proofs, and hence the statement holds for all $\ell\le L$.

Taking $\alpha=D_L=1+2L$, the Pfaffian-chain condition of Definition~\ref{def:pfaff} is satisfied, because for each list element
$f_i$ and each coordinate index $p$ we have expressed $\partial_p f_i$ as a polynomial in $(x,f_1,\dots,f_i)$:
indeed, \eqref{eq:ds-proof-rewrite} expresses $\partial_p s^{(\ell)}_k$ using only functions from layers before $\ell$, which appear
earlier in the list, while \eqref{eq:jet1-proof-rewrite}--\eqref{eq:jet2-proof-rewrite} express $\partial_p u^{(\ell)}_{k,q}$
using $s^{(\ell)}_k$ and functions $u^{(\ell)}_{k,q'}$ with $q'>q$, all of which appear earlier within the same block. This finishes the proof of the Sublemma. \oprocend

It remains to represent the network output as a polynomial in the list. By \eqref{eq:nn-output},
\[
F(x)=c_0+c^\top h^{(L)}(x)
=c_0+\sum_{k=1}^{n_L}c_k\,h^{(L)}_k(x)
=c_0+\sum_{k=1}^{n_L}c_k\,u^{(L)}_{k,0}(x),
\]
which is a polynomial of total degree $1$ in the list elements. Thus one may take $\beta=1$ in the sense of
Definition~\ref{def:pfaff}. Combining the length $R$, the degree bound $\alpha=1+2L$, and $\beta=1$ gives the format
\eqref{eq:format-proof}.
\end{proof}

\begin{proof}[Proof of Proposition~\ref{prop:deep_zeros_main}]
Assume $d=1$ and let $I\subset\mathbb R$ be an open interval on which $F$ is real-analytic.
By Proposition~\ref{prop:deep-pfaff} with $d=1$, the restriction of $F$ to $I$ is Pfaffian of format
\[
(1,R,\alpha,\beta)=\Big(1,\ (r+2)\sum_{\ell=1}^{L}n_\ell,\ 1+2L,\ 1\Big),
\]
so $R$ agrees with \eqref{eq:R}. If $F\not\equiv 0$ on $I$, consider the zero set
\[
Z:=\{x\in I:\ F(x)=0\}.
\]
Since $F$ is Pfaffian on $I$, the set $Z$ is semi-Pfaffian on $I$ described by a single Pfaffian equality. Moreover, because $F$ is
real-analytic and not identically zero, $Z$ is finite, and therefore
\[
\Zeros(F;I)=|Z|=b_0(Z)=\Betti(Z).
\]
Applying Theorem~\ref{thm:pf-betti} to $Z$ (with $d=1$ and $s=1$) yields
\[
\Zeros(F;I)
\le
2^{\frac{R(R-1)}{2}}\; C_I\;\Big(\beta+\min\{1,R\}\alpha\Big)^{1+R}
=
2^{\frac{R(R-1)}{2}}\; C_I\;(\alpha+\beta)^{1+R}.
\]
Substituting $\alpha=1+2L$ and $\beta=1$ gives
\[
\Zeros(F;I)
\le
2^{\frac{R(R-1)}{2}}\; C_I\;\bigl(2+2L\bigr)^{1+R}
=
2^{\frac{R(R+1)}{2}+1}\; C_I\;(1+L)^{1+R}.
\]
Absorbing the constant factor $2$ into $C_I$ yields
\[
\Zeros(F;I)\ \le\ 2^{\frac{R(R+1)}{2}}\,C_I\,(1+L)^{R+1}.
\]
If $F\equiv 0$ on $I$, then $\mathcal S_{\ge 0}(F)=I$ and the interval-component bound is trivial.

Now, assume $F\not\equiv 0$ and let $z_1<\cdots<z_N$ be the distinct zeros of $F$ in $I$, where $N:=\Zeros(F;I)$.
Then $I$ is partitioned into the at most $N+1$ open intervals
\[
I\cap(-\infty,z_1),\quad (z_1,z_2),\quad \dots,\quad (z_{N-1},z_N),\quad I\cap(z_N,\infty).
\]
On each such subinterval $F$ has no zeros and hence has constant sign; therefore $\{x\in I:\ F(x)\ge 0\}$ restricted to that
subinterval is either empty or the entire subinterval. Consequently, $\mathcal S_{\ge 0}(F)=\{x\in I:\ F(x)\ge 0\}$ is a union of at most $N+1$ intervals. Absorbing the additive constant into $C_I$ and substituting the above bound on $N$ gives the stated interval bound.
\end{proof}

\begin{proof}[Proof of Theorem~\ref{thm:nn_betti_s1}]
Let $V\subset\mathbb R^d$ be an open domain on which $F$ is real-analytic. By Proposition~\ref{prop:deep-pfaff}, $F$ is Pfaffian on $V$
with format
\[
(d,R,\alpha,\beta)=\Big(d,\ (r+2)\sum_{\ell=1}^{L}n_\ell,\ 1+2L,\ 1\Big).
\]
By Definition~\ref{def:semipf}, the decision region
\[
\mathcal S_{\ge 0}(F)=\{x\in V:\ F(x)\ge 0\}
\]
is semi-Pfaffian on $V$ described by a single Pfaffian inequality. Applying Theorem~\ref{thm:pf-betti} with $s=1$ yields
\[
\Betti\!\big(\mathcal S_{\ge 0}(F)\big)\ \le\ \mathcal C_V(d,1,R,\alpha,\beta)
=
2^{\frac{R(R-1)}{2}}\; C_V\; \Big(d\beta+\min\{d,R\}\alpha\Big)^{d+R}.
\]
Since $\beta=1$ and $\alpha=1+2L$, we obtain
\[
\Betti\!\big(\mathcal S_{\ge 0}(F)\big)
\le
2^{\frac{R(R-1)}{2}}\; C_V\; \Big(d+\min\{d,R\}(1+2L)\Big)^{d+R}
=
\mathfrak B_V(d,R,L),
\]
which is the claimed bound. In particular, the number of connected components satisfies 
\[
b_0(\mathcal S_{\ge 0}(F))\le \Betti(\mathcal S_{\ge 0}(F))\le \mathfrak B_V(d,R,L).
\]
\end{proof}

The proof of Theorem~\ref{thm:nn_rankdrop_main} relies on some preliminary results, which we establish next. We start with
the following closure properties of Pfaffian functions, which follow directly from Definition~\ref{def:pfaff} and the
chain rule. We provide a proof for completeness. 
\begin{lemma}
\label{lem:pfaff-closure}
Let $V\subset\mathbb R^d$ be open.
\begin{enumerate}
\item If $G_1,\dots,G_N$ are Pfaffian on $V$, then any polynomial combination $P(G_1,\dots,G_N)$ is Pfaffian on $V$.
\item If $G$ is Pfaffian on $V$, then each partial derivative $\partial_{x_p}G$ is Pfaffian on $V$, after possibly extending the
underlying Pfaffian chain by adjoining finitely many auxiliary Pfaffian functions.
\end{enumerate}
\end{lemma}

\begin{proof}
We prove (1) and (2) in order.

\emph{Proof of}~(1).
Choose a Pfaffian chain $(f_1,\dots,f_R)$ on $V$ with respect to which each $G_j$ can be written as
\[
G_j(x)=Q_j\bigl(x,f_1(x),\dots,f_R(x)\bigr).
\]
If the functions $G_j$ are given with respect to different Pfaffian chains, we first concatenate those chains into a single chain
and rewrite each $G_j$ accordingly. Then $P(G_1,\dots,G_N)$ is a polynomial in $(x,f_1,\dots,f_R)$, hence Pfaffian by
Definition~\ref{def:pfaff}.

\emph{Proof of}~(2).
Let $(f_1,\dots,f_R)$ be a Pfaffian chain on $V$ and let $Q$ be a polynomial such that
$
G(x)=Q\bigl(x,f_1(x),\dots,f_R(x)\bigr)
$.
Differentiating and applying the chain rule yields
\[
\partial_{x_p}G
=
\frac{\partial Q}{\partial x_p}(x,f(x))
+\sum_{i=1}^{R}\frac{\partial Q}{\partial y_i}(x,f(x))\,\partial_{x_p}f_i(x),
\qquad f(x):=\bigl(f_1(x),\dots,f_R(x)\bigr).
\]
By the Pfaffian-chain property, for each $i$ we have
\[
\partial_{x_p}f_i(x)=P_{i,p}\bigl(x,f_1(x),\dots,f_i(x)\bigr)
\]
for some polynomial $P_{i,p}$. Substituting these expressions shows that $\partial_{x_p}G$ is a polynomial in
$(x,f_1,\dots,f_R)$, hence Pfaffian on $V$. If needed, we may extend the chain by adjoining the finitely many auxiliary Pfaffian
functions arising in this polynomial expression.
\end{proof}

\begin{lemma}
\label{lem:bracket-pfaff}
Let $V\subset\mathbb R^d$ be open and let $X,Y$ be real-analytic vector fields on $V$ written as
\[
X(z)=\sum_{a=1}^{d} X_a(z)\,\partial_{z_a},
\qquad
Y(z)=\sum_{a=1}^{d} Y_a(z)\,\partial_{z_a},
\]
where each coefficient function $X_a,Y_a:V\to\mathbb R$ is Pfaffian on $V$. Then each coefficient function of the Lie bracket
$[X,Y]$ is Pfaffian on $V$.
\end{lemma}

\begin{proof}
For each $a\in\{1,\dots,d\}$, the $a$-th coefficient of $[X,Y]$ is given by
\[
[X,Y]_a
=
\sum_{b=1}^{d}\Bigl(X_b\,\partial_{z_b}Y_a - Y_b\,\partial_{z_b}X_a\Bigr).
\]
Thus $[X,Y]_a$ is obtained from the Pfaffian functions $\{X_b,Y_b\}_{b=1}^{d}$ by finitely many sums and products, together with
the partial derivatives $\partial_{z_b}X_a$ and $\partial_{z_b}Y_a$. By Lemma~\ref{lem:pfaff-closure}, Pfaffian functions are
closed under sums, products, and partial differentiation (after possibly extending the underlying chain). Therefore each
$[X,Y]_a$ is Pfaffian on $V$.
\end{proof}

\begin{proof}[Proof of Theorem~\ref{thm:nn_rankdrop_main}]
By assumption, for each $i\in\{1,\dots,m\}$ and $p\in\{1,\dots,d\}$ the coefficient function $X_{i,p}:V\to\mathbb R$ is realized on
$V$ by a feedforward network of fixed architecture $(d,L,n_1,\dots,n_L)$ with activation $\sigma\in\Aquad{r}$, and is real-analytic
on $V$. Invoking Proposition~\ref{prop:deep-pfaff} with input dimension $d$ for each such network, each $X_{i,p}$ is Pfaffian on $V$
with a format bound depending only on $(d,L,n_1,\dots,n_L)$ and $r$, and not on the weights.

Let us now fix $k\in\mathbb N$. Note that every bracket $Y\in\mathcal B_k$ is obtained by iterating the Lie bracket operation from
$X_1,\dots,X_m$ at most $k$ times. Therefore, repeated application of Lemma~\ref{lem:bracket-pfaff} shows that every coefficient
function of every $Y\in\mathcal B_k$ is Pfaffian on $V$, and hence every entry of the matrix $A_k(z)$ is Pfaffian on $V$.

Let us now fix $\rho\in\{0,1,\dots,d\}$, and denote by $M_1,\dots,M_{s_{k,\rho}}$ the list of all $(\rho+1)\times(\rho+1)$ minors of
$A_k$, where
\[
s_{k,\rho}:=\binom{d}{\rho+1}\binom{|\mathcal B_k|}{\rho+1}.
\]
Each minor $M_i(z)$ is the determinant of a submatrix of $A_k(z)$, hence a polynomial in the entries of $A_k(z)$; by
Lemma~\ref{lem:pfaff-closure}(1), each $M_i$ is Pfaffian on $V$.

For any real matrix $A$, $\rank A\le \rho$ if and only if all $(\rho+1)\times(\rho+1)$ minors vanish. Applying this to $A=A_k(z)$
gives the real-analytic description
\[
\locus{k}{\rho}
=
\{z\in V:\ \rank A_k(z)\le \rho\}
=
\bigcap_{i=1}^{s_{k,\rho}}\{z\in V:\ M_i(z)=0\}.
\]
Thus $\locus{k}{\rho}$ is semi-Pfaffian on $V$, as it is represented by $s_{k,\rho}$ Pfaffian equalities. Applying
Theorem~\ref{thm:pf-betti} yields
\begin{align*}
\Betti(\locus{k}{\rho})\ &\le\
\mathfrak R_V\!\Big(d,m,k,\rho;\,L,n_1,\dots,n_L,r\Big)\\
&:=
2^{\frac{R_k(R_k-1)}{2}}\; (s_{k,\rho})^{d}\; C_V\;\Big(d\beta_k+\min\{d,R_k\}\alpha_k\Big)^{d+R_k},
\end{align*}
for any uniform Pfaffian format bound $(d,R_k,\alpha_k,\beta_k)$ valid for the minors $M_i$, which proves the claim.
\end{proof}

\providecommand{\bysame}{\leavevmode\hbox to3em{\hrulefill}\thinspace}
\providecommand{\MR}{\relax\ifhmode\unskip\space\fi MR }
\providecommand{\MRhref}[2]{%
  \href{http://www.ams.org/mathscinet-getitem?mr=#1}{#2}
}
\providecommand{\href}[2]{#2}

\end{document}